\documentclass[a4paper,10pt]{article}

\usepackage[utf8x]{inputenc}
\usepackage{amsmath}
\usepackage{amssymb}
\usepackage{amsthm}
\usepackage[dvips]{graphicx}
\usepackage{cite}
\usepackage{pgfplots}
\usepackage{geometry}
\usepackage[authoryear, round]{natbib}
\usepackage{amsfonts}
\newtheorem{definition}{Definition}

\newtheorem{theorem}{Theorem}

\begin{document}

\title{Marginal and simultaneous predictive classification using stratified graphical models}
\author{Henrik Nyman$^{1, \ast}$, Jie Xiong$^{2}$, Johan Pensar$^{1}$, Jukka Corander$^{2}$ \\
$^{1}$Department of Mathematics and Statistics, \AA bo Akademi University, Finland \\
$^{2}$Department of Mathematics and Statistics, University of Helsinki, Finland \\
$^{\ast}$Corresponding author, Email: hennyman@abo.fi}
\date{}

\maketitle

\begin{abstract}
An inductive probabilistic classification rule must generally obey the principles of Bayesian predictive inference, such that all observed and unobserved stochastic quantities are jointly modeled and the parameter uncertainty is fully acknowledged through the posterior predictive distribution. Several such rules have been recently considered and their asymptotic behavior has been characterized under the assumption that the observed features or variables used for building a classifier are conditionally independent given a simultaneous labeling of both the training samples and those from an unknown origin. Here we extend the theoretical results to predictive classifiers acknowledging feature dependencies either through graphical models or sparser alternatives defined as stratified graphical models. We also show through experimentation with both synthetic and real data that the predictive classifiers based on stratified graphical models have consistently best accuracy compared with the predictive classifiers based on either conditionally independent features or on ordinary graphical models.
\end{abstract}


\section{Introduction}
\label{sec:intro}
Supervised classification is one of the most common tasks considered in machine learning and statistics \citep{Ripley96, Duda00, Hastie09, Bishop07pattern}, with a wide variety of applications over practically all fields of science and engineering. Today, there exists a myriad of different classification methods, out of which those based on probabilistic models are widely accepted as the most sensible way to solve classification problems. Probabilistic methods are often themselves classified as either generative or discriminative, depending on whether one directly models the class posterior distribution (discriminative classifiers) or first the joint distribution of observed features (variables) conditional on class training data and then the posterior distribution of labels is obtained through Bayes' rule. There has been a debate around which of these approaches should be preferred in a particular application, see \cite{Ripley96}, \cite{Hastie09}, \cite{Bishop07pattern}, and \cite{Pernkopf05}, however, both classes of methods continue to be supported and further developed. One of the popular methods of probabilistic classification is based on encoding feature dependencies with Bayesian networks \citep{Friedman97}. Such models can often represent data structures more faithfully than the naive Bayes classifier, which has been shown to yield dramatic improvements in classification accuracy in some cases. Numerous variants and extensions of the original framework introduced by \cite{Friedman97} have been considered over the years, e.g. \cite{Keogh99}, \cite{Pernkopf05}, \cite{Su06}, \cite{Cerquides05}, \cite{Madden09}, and \cite{Holmes08}. \cite{Friedman97} concluded that general Bayesian networks did not perform better than the naive Bayes classifier, however, later \cite{Madden09} showed that this suboptimal behavior was attributable to the maximum likelihood estimation of the parameters used by \cite{Friedman97} and when the parameter estimates were smoothed with a prior, the classification accuracy of the models was dramatically improved. 

Albeit the above mentioned classifiers are occasionally called predictive, they are not predictive methods in the sense of \cite{Geisser64, Geisser66, Geisser93}, who considered the foundations of general Bayesian predictive inference. Truly predictive generative classifiers need typically to model also the joint predictive distribution of the features, which leads to an infinite mixture over the parameter space when uncertainty about generating model parameters is characterized through their posterior distribution. In addition, as shown by \cite{Corander13c, Corander13a, Corander13b}, depending on the loss function employed for the classification task, genuinely inductive predictive classifiers may also require that all the data are predictively classified in a simultaneous fashion. This is in contrast with the standard classification methods which beyond the training data handle each sample independently and separately from others, which was termed marginal classification in \cite{Corander13a}. Simultaneous classifiers are therefore computationally much more demanding, because they necessitate modeling of the joint posterior-predictive distribution of the unknown sample labels. 

It appears that the theory of simultaneous predictive classification is not widely known in the general statistical or machine learning literature. To the best of our knowledge, none of the Bayesian network, or more generally graphical model based classifiers introduced earlier, are strictly Bayesian predictive in the meaning of \cite{Geisser93}. However, in speech recognition the theoretical optimality of predictive simultaneous classifiers was notified already by \cite{Nadas85}. Later work has demonstrated their value in several speech recognition applications, see, e.g. \cite{Huo00} and \cite{Maina11}. Also, \cite{Ripley88} discussed the enhanced performance of simultaneous classifiers for statistical image analysis, although not in the posterior predictive sense. 

\cite{Corander13c, Corander13a, Corander13b} considered the situation where features are assumed conditionally independent, given a joint simultaneous labeling of all the samples of unknown origin that are to be classified. Even if these samples were generated independently from the same underlying distributions, their labels are not in general independent in the posterior predictive distribution. Here we extend the inductive classifier learning to a situation where the feature dependencies are encoded either by ordinary graphical models, or by a recently introduced class of sparser stratified graphical models. We show that the results of \cite{Corander13a}, concerning the asymptotic equality of simultaneous and marginal classifiers when the amount of training data tends to infinity, generalize to the situation with an arbitrary Markov network structure for the features in each class. Moreover, it is also shown that the asymptotic equality holds between graphical and stratified graphical models as well. For finite training data, we demonstrate that considerable differences in classification accuracy may arise between predictive classifiers built under the assumptions of empty graphs (predictive naive Bayes classifier), ordinary graphical models and stratified graphical models. The sparse stratified graphical models tend to consistently yield the best performance in our experiments. 

The remainder of this article is structured as follows. In Section \ref{sec:SGM} we give a short introduction to theory involving graphical models and stratified graphical models. Section \ref{sec:ML} contains the theory needed to calculate the marginal likelihood of a dataset given an SGM, to be used in Section \ref{sec:classifier} to define marginal and simultaneous SGM classifiers. In Section \ref{sec:results} these novel types of classifiers are compared to classifiers utilizing a GM structure as well as to the predictive naive Bayes classifier using a range of synthetic and real datasets. Some general remarks and comments are given in the last section while proofs of theorems and certain technical details are provided in the Appendix and through online supplementary materials.

\section{Stratified graphical models}
\label{sec:SGM}
In this section we give a short introduction to graphical models (GMs) and in particular stratified graphical models (SGMs). For a comprehensive account of the statistical and computational theory of these models, see \cite{Whittaker90}, \cite{Lauritzen96}, \cite{Koller09}, and \cite{Nyman13}.

Let $G(\Delta,E)$, be an undirected graph, consisting of a set of nodes $\Delta$ and of a set of undirected edges $E\subseteq\{\Delta \times \Delta\}$. For a subset of nodes $A \subseteq \Delta$, $G_{A}=G(A,E_{A})$ is a subgraph of $G$, such that the nodes in $G_{A}$ are equal to $A$ and the edge set comprises those edges of the original graph for which both nodes are in $A$, i.e. $E_{A} = \{A \times A\} \cap E$. Two nodes $\gamma$ and $\delta$ are \textit{adjacent} in a graph if $\{\gamma, \delta\}\in E$, that is an edge exists between them. A \textit{path} in a graph is a sequence of nodes such that for each successive pair within the sequence the nodes are adjacent. A \textit{cycle} is a path that starts and ends with the same node. A \textit{chord} in a cycle is an edge between two non-consecutive nodes in the cycle. Two sets of nodes $A$ and $B$ are said to be \textit{separated} by a third set of nodes $S$ if every path between nodes in $A$ and nodes in $B$ contains at least one node in $S$. A graph is defined as \textit{complete} when all pairs of nodes in the graph are adjacent.

A graph is defined as \textit{decomposable} if there are no chordless cycles containing four or more unique nodes. A \textit{clique} in a graph is a set of nodes $C$ such that the subgraph $G_{C}$ is complete and there exists no other set
$C^*$ such that $C \subset C^*$ and $G_{C^*}$ is also complete. The set of cliques in the graph $G$ will be denoted by $\mathcal{C}(G)$. The set of \textit{separators}, $\mathcal{S}(G)$, in the decomposable graph $G$ can be obtained through intersections of the cliques of $G$ ordered in terms of a junction tree, see e.g. \cite{Golumbic04}. 

Associating each node $\delta \in \Delta$ with a stochastic feature, or equivalently variable, $X_{\delta}$, a GM is defined by the pair $G=G(\Delta,E)$ and a joint distribution $P_{\Delta}$ over the variables $X_{\Delta}$ satisfying a set of restrictions induced by $G$. In the remainder of the text we use the terms \textit{feature} and \textit{variable} interchangeably. The outcome space for the variables $X_A$, where $A \subseteq \Delta$, is denoted by $\mathcal{X}_{A}$ and an element in this space by $x_{A} \in \mathcal{X}_{A}$. It is assumed throughout this paper that all considered variables are binary. However, the introduced theory can readily be extended to categorical discrete variables with larger than dichotomous outcome spaces. 

Given the graph of a GM, it is possible to ascertain if two sets of random variables $X_A$ and $X_B$ are marginally or conditionally independent. If there exists no path from a node in $A$ to a node in $B$ the two sets of variables are marginally independent, i.e. $P(X_A, X_B) = P(X_A) P(X_B)$. Similarly the variables $X_A$ and $X_B$ are conditionally independent given a third set of variables $X_S$, $P(X_A, X_B \mid X_S) = P(X_A \mid X_S) P(X_B \mid X_S)$ , if $S$ separates $A$ and $B$ in $G$. In addition to marginal and conditional independencies, SGMs allow for the introduction of context-specific independencies. Using SGMs, two variables $X_{\delta}$ and $X_{\gamma}$ may be independent given that a specific set of variables $X_A$ assume a certain outcome $x_A$, i.e. $P(X_{\delta}, X_{\gamma} \mid X_A = x_A) = P(X_{\delta} \mid X_A = x_A) P(X_{\gamma} \mid X_A = x_A)$. The set of outcomes for which such a context-specific independence holds is referred to as a \textit{stratum}.

\begin{definition}[Stratum]
\label{stratum}
Let the pair $(G, P_{\Delta})$ be a GM. For all $\{\delta, \gamma\} \in E$, let $L_{\{\delta, \gamma\}}$ denote the set of nodes adjacent
to both $\delta$ and $\gamma$. For a non-empty $L_{\{\delta, \gamma\}}$, define the stratum of the edge  $\{\delta, \gamma\}$ as the subset $\mathcal{L}_{\{\delta, \gamma\}}$ of outcomes $x_{L_{\{\delta, \gamma\}}} \in \mathcal{X}_{L_{\{\delta, \gamma\}}}$ for which $X_{\delta}$ and $X_{\gamma}$ are independent given $X_{L_{\{\delta, \gamma\}}} = x_{L_{\{\delta, \gamma\}}}$, i.e. $\mathcal{L}_{\{\delta, \gamma\}} = \{ x_{L_{\{\delta, \gamma\}}} \in  \mathcal{X}_{L_{\{\delta, \gamma\}}} : X_{\delta} \perp X_{\gamma} \mid X_{L_{\{\delta, \gamma\}}} =  x_{L_{\{\delta, \gamma\}}} \}$.
\end{definition}
\begin{figure}[htb]
\begin{center}
\includegraphics{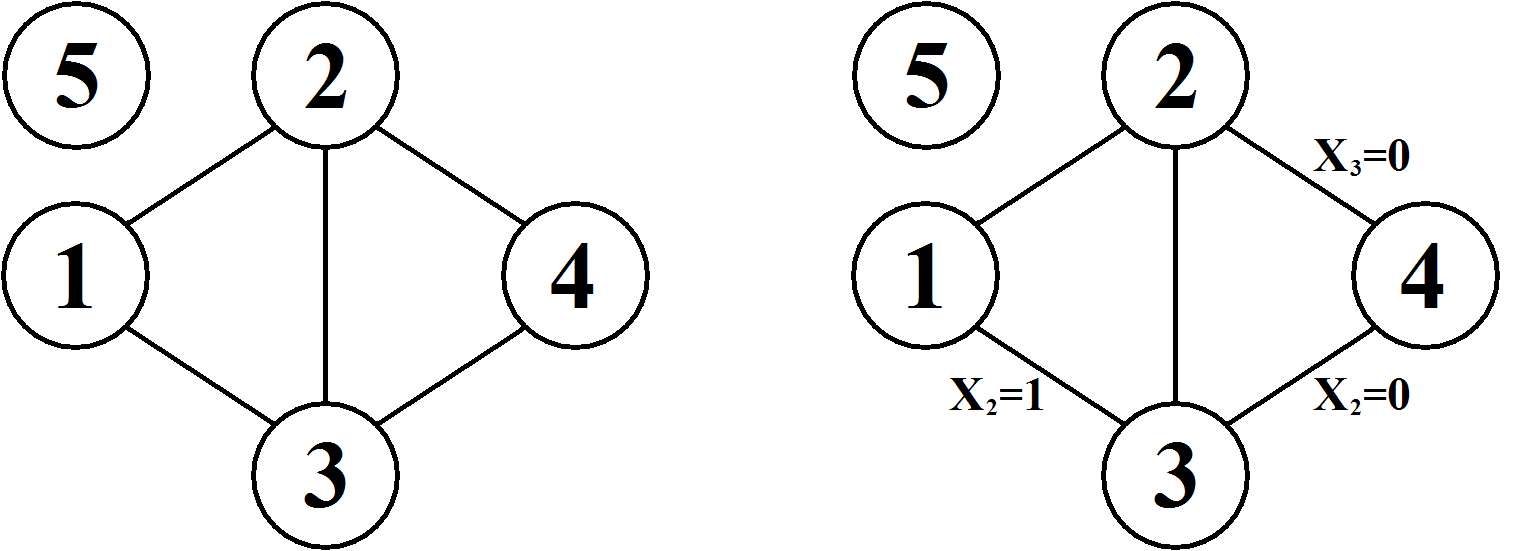}
\end{center}
\caption{In (\textbf{a}) a graphical model and in (\textbf{b}) a stratified graphical model.}
\label{fig:GMSGM}
\end{figure}
A stratum can be represented graphically by adding conditions to an edge in a graph as shown in Figure \ref{fig:GMSGM}b. The graph in Figure \ref{fig:GMSGM}a induces both marginal and conditional independencies, for instance $X_1 \perp X_5$ and $X_1 \perp X_4 \mid X_2, X_3$. In addition the graph in Figure \ref{fig:GMSGM}b induces the context-specific independencies $X_1 \perp X_3 \mid X_2 = 1$, $X_2 \perp X_4 \mid X_3 = 0$, and $X_3 \perp X_4 \mid X_2 = 0$.

\begin{definition} [Stratified graphical model]
An SGM is defined by the triple $(G, L, P_{\Delta})$, where $G$ is the underlying graph, $L$ equals the joint collection of all strata $\mathcal{L}_{\{\delta,\gamma\}}$ for the edges of $G$, and $P_{\Delta}$ is a joint distribution on $\Delta$ which factorizes according to the restrictions imposed by $G$ and $L$.
\end{definition}
The pair $(G, L)$ consisting of the graph $G$ with the stratified edges (edges associated with a stratum) determined by $L$ will be referred to as a stratified graph (SG), usually denoted by $G_L$. When the collection of strata, $L$, is empty, $G_{L}$ equals $G$.

Given a decomposable graph the marginal likelihood of a dataset can be calculated analytically. To extend this property to SGs, we introduce the concept \textit{decomposable} SG. Consider an SG with a decomposable underlying graph $G$ having the cliques $\mathcal{C}(G)$ and separators $\mathcal{S}(G)$. The SG is defined as decomposable if no strata are assigned to edges in any separator and in every clique all stratified edges have at least one node in common.
\begin{definition}[Decomposable SG]
Let $(G, L)$ constitute an SG with $G$ being decomposable. Further, let $E_{L}$ denote the set of all stratified edges, $E_{C}$ the set of all edges in clique $C$, and $E_{\mathcal{S}}$ the set of all edges in the separators of $G$. The SG is defined as decomposable if
\[
E_{L}\cap E_{\mathcal{S}}=\emptyset,
\]
and
\[
E_{L} \cap E_{C} = \emptyset \hspace{0.4cm} \text{ or } \hspace{0.0cm} \bigcap_{\{\delta,\gamma\}\in E_{L}\cap E_{C}} \hspace{-0.5cm} \{\delta,\gamma\} \hspace{0.1cm} \neq \hspace{0.1cm} \emptyset \hspace{0.1cm} \text{ for all } \hspace{0.1cm} C \in \mathcal{C}(G).
\]
\end{definition}
An SGM where $(G, L)$ constitutes a decomposable SG is termed a decomposable SGM. The SG in Figure \ref{fig:GMSGM}b is an example of a decomposable SG.

\section{Calculating the marginal likelihood of a dataset given an SGM}
\label{sec:ML}
Let $\mathbf{X}$ denote a data matrix consisting of $n$ binary vectors, each containing $|\Delta|$ elements. Using $\mathbf{X}_{A}$ we denote the subset of $\mathbf{X}$ for the variables in $A$. For an arbitrary decomposable graph $G$, under a prior distribution which enjoys the hyper-Markov property \citep{Dawid93}, the marginal likelihood of the dataset $\textbf{X}$ factorizes as
\begin{equation}
\label{eq:cs}
P(\mathbf{X} \mid G)=\frac{\prod_{C\in\mathcal{C}(G)}P_{C}(\mathbf{X}_{C})}
{\prod_{S\in\mathcal{S}(G)}P_{S}(\mathbf{X}_{S})},
\end{equation}
where $\mathcal{C}(G)$ and $\mathcal{S}(G)$ are the cliques and separators, respectively, of $G$. For any subset $A \subseteq \Delta$ of nodes, $P_{A}(\mathbf{X}_{A})$ denotes the marginal likelihood of the subset $\mathbf{X}_{A}$ of data. \cite{Nyman13} showed that this factorization can also be applied for decomposable SGs.

\cite{Nyman13} derived a formula for calculating $P_{C}(\mathbf{X}_{C})$ and $P_{S}(\mathbf{X}_{S})$, which is applicable to both ordinary graphs and SGs. Their derivation is based on introducing a specific ordering of the clique variables and merging some conditional distributions, for an example illustrating this procedure see Appendix A. The result is a modified version of the formula introduced by \cite{Cooper92} for the marginal likelihood of a Bayesian network, leading to the following expression for the clique marginal likelihood
\begin{equation}
\label{eq:ML}
P_{C}(\mathbf{X}_{C}) = \prod_{j = 1}^{d} \prod_{l = 1}^{q_{j}} \frac{\Gamma(\sum_{i = 1}^{k_{j}} \alpha_{jil})}{\Gamma(n(\pi_{j}^{l}) + \sum_{i = 1}^{k_{j}} \alpha_{jil})} \prod_{i = 1}^{k_{j}} \frac{\Gamma(n(x_{j}^{i} \mid \pi_{j}^{l}) + \alpha_{jil})}{\Gamma(\alpha_{jil})},
\end{equation}
where $\Gamma$ denotes the gamma function, $d$ equals the number of variables in the clique $C$, $q_{j}$ is the number of \textit{distinguishable} parent combinations for variable $X_j$ (i.e. there are $q_j$ distinct conditional distributions for variable $X_j$), $k_{j}$ is the number of possible outcomes for variable $X_j$, $\alpha_{jil}$ is the hyperparameter used in a Dirichlet prior distribution corresponding to the outcome $i$ of variable $X_j$ given that the outcome of the parents of $X_j$ belongs to group $l$, $n(\pi_{j}^{l})$ is the number of observations of the combination $l$ for the parents of variable $X_j$, and finally, $n(x_{j}^{i} \mid \pi_{j}^{l})$ is the number of observations where the outcome of variable $X_j$ is $i$ given that the observed outcome of the parents of $X_j$ belongs to $l$. Note that in this context a parent configuration $l$ is not necessarily comprised of a single outcome of the parents of variable $X_j$, but rather a \textit{group} of outcomes with an equivalent effect on $X_j$ (see Appendix A). 

The hyperparameters of the Dirichlet prior distribution can be chosen relatively freely, here we use 
\[
\alpha_{jil}=\frac{N\cdot\lambda_{jl}}{\pi_{j}\cdot k_{j}},
\]
where $N$ is the equivalent sample size, $\pi_{j}$ is the total number of possible outcomes for the parents of variable $X_j$ ($= 1$ for $X_1$) and $k_{j}$ is the number of possible outcomes for variable $X_j$. Further, $\lambda_{jl}$ equals the number of outcomes for the parents of variable $X_j$ in \textit{group} $l$ with an equivalent effect on $X_j$, if $X_j$ is the last variable in the ordering. Otherwise, $\lambda_{jl}$ equals one. Using \eqref{eq:ML} the values $P_{S}(X_{S})$ can also be calculated, as can $P_{C}(X_{C})$ and $P_{S}(X_{S})$ for ordinary GMs. For these instances each group $l$ consists of a single outcome of the parents of variable $X_j$.

\section{Predictive SGM Classifier}
\label{sec:classifier}
SGMs are now employed to define a novel type of predictive classifier, which acknowledges dependencies among variables but can also encode additional sparsity when this is supported by the training data. In comparison to GMs, SGMs allow for a more detailed representation of the dependence structure, thus enhancing the classification process. We assume that the dependence structure can freely vary across different classes. Let $\mathbf{X}^R$, consisting of $m$ observations on $|\Delta|$ features, constitute the training data over $K$ classes. The class labels for the observations in $\mathbf{X}^R$ are specified by the vector $R$, where the element $R(i) \in \{1,...,K\}$ defines the class of observation $i,i=1,...,m$. Similarly, $\mathbf{X}^T$ represents the test data consisting of $n$ observations, and their classification is determined by the vector $T$, which is the main target of predictive inference. Using the training data for class $k$, $\mathbf{X}^{R, k}$, a search for the SGM with optimal marginal likelihood, $G_L^k$, can be conducted \citep{Nyman13}. Given the resulting SGs for each class, $G_L^{A}$, posterior predictive versions of the equations \eqref{eq:cs} and \eqref{eq:ML} can be used to probabilistically score any candidate classification of the test data. 

We consider two types of predictive classifiers, a marginal classifier and a simultaneous classifier \citep{Corander13a}. Both assign a predictive score to the label vector $T$, which can be normalized into a posterior given a prior distribution over possible labellings. \cite{Corander13a} introduce formally various classification rules using the posterior distribution and decision theory, however, here we simply consider the \textit{maximum a posteriori} (MAP) rule under a uniform prior, which corresponds to maximization of the predictive score function. For the marginal classifier the predictive score, $P_{\text{mar}}$, is calculated as
\[
\begin{split}
P_{\text{mar}}(T \mid \mathbf{X}^T, \mathbf{X}^R, R, G_L^A) &= \prod_{i=1}^n P(\mathbf{X}_i^T \mid T, \mathbf{X}^R, R, G_L^A) \\ &=
\prod_{i=1}^n \frac{\prod_{C \in \mathcal{C}(G_L^{T(i)})}P_{C}(\mathbf{X}_{i, C}^{T} \mid \mathbf{X}^{R, T(i)})}{\prod_{S \in \mathcal{S}(G_L^{T(i)})}P_{S}(\mathbf{X}_{i, S}^{T} \mid \mathbf{X}^{R, T(i)})},
\end{split}
\]
where $\mathbf{X}_i^T$ denotes the $i$th observation in the test data, and $\mathbf{X}_{i, C}^{T}$ denotes the outcomes of the variables associated to clique $C$ in this observation. The posterior predictive likelihoods $P_{C}(\mathbf{X}_{i, C}^{T} \mid \mathbf{X}^{R, T(i)})$ and $P_{S}(\mathbf{X}_{i, S}^{T} \mid \mathbf{X}^{R, T(i)})$ are calculated using \eqref{eq:ML} and the updated hyperparameters $\beta_{jil}$ instead of $\alpha_{jil}$,
\begin{equation}
\label{eq:beta}
\beta_{jil} = \alpha_{jil} + m(x_{j}^{i} \mid \pi_{j}^{l}),
\end{equation}
where $m(x_{j}^{i} \mid \pi_{j}^{l})$ is the number of observations in $\mathbf{X}^{R, T(i)}$ where the outcome of variable $X_j$ is $i$ given that the observed outcome of the parents of $X_j$ belongs to group $l$. 

The optimal classification decision is obtained by the vector $T$ that maximizes the score function over all $n$ samples in the test data, i.e. 
\[
\arg \mathop{\max}_{T} P_{\text{mar}}(T \mid \mathbf{X}^T, \mathbf{X}^R, R, G_L^A).
\]
Using the simultaneous predictive classifier the observations in the test data are not classified independently of each other as is the case with the marginal classifier presented above. Instead, the score function becomes
\[
\begin{split}
P_{\text{sim}}(T \mid \mathbf{X}^T, \mathbf{X}^R, R, G_L^A) &= \prod_{k=1}^K P(\mathbf{X}^{T, k} \mid T, \mathbf{X}^R, R, G_L^A) \\ &=
\prod_{k=1}^K \frac{\prod_{C \in \mathcal{C}(G_L^k)}P_{C}(\mathbf{X}_{C}^{T, k} \mid \mathbf{X}^{R, k})}{\prod_{S \in \mathcal{S}(G_L^k)}P_{S}(\mathbf{X}_{S}^{T, k} \mid \mathbf{X}^{R, k})},
\end{split}
\]
where $\mathbf{X}^{T, k}$ denotes observation in the test data assigned to class $k$ by $T$, and $\mathbf{X}_{C}^{T, k}$ denotes the outcomes of the variables associated to clique $C$ for these observations. The posterior predictive likelihoods $P_{C}(\mathbf{X}_{C}^{T, k} \mid \mathbf{X}^{R, k})$ and $P_{S}(\mathbf{X}_{S}^{T, k} \mid \mathbf{X}^{R, k})$ are again calculated using \eqref{eq:ML} and the updated hyperparameters $\beta_{jil}$. The optimal labeling is obviously still determined by the vector $T$ that optimizes the simultaneous score function. 

Intuitively, the simultaneous classifier merges information from the test data already assigned to class $k$ with the training data of class $k$, when assessing the joint probability of observing them all. Depending on the level of complexity of the model and the size of the training and test data sets, this increases the accuracy of the classifier, as shown in \cite{Corander13a}. However, the theorems below formally establish that when the size of the training data grows, the classification decisions based on marginal and simultaneous predictive SGM classifiers become equivalent, and further that also GM and SGM based predictive classifiers become equivalent.
\begin{theorem}
\label{th:marSim}
The marginal and simultaneous predictive SGM classifiers are asymptotically equivalent as the size of the training data goes to infinity.
\end{theorem}
\begin{proof}[Theorem \ref{th:marSim}]
See Appendix B.
\end{proof}

\begin{theorem}
\label{th:GMSGM}
The predictive SGM and GM classifiers are asymptotically equivalent as the size of the training data goes to infinity.
\end{theorem}
\begin{proof}[Theorem \ref{th:GMSGM}]
See Appendix C.
\end{proof}

The vector $T$ that optimizes the predictive score is identified using the same methods as in \cite{Corander13a, Corander13b}. For the marginal classifier we cycle through all observations in the test data assigning each one to the class optimizing the predictive score $P_{\text{mar}}$. For the simultaneous classifier we begin by initializing a start value for $T$, this vector may be generated randomly or, for instance, chosen as the vector that maximizes the score for the marginal classifier. The elements in $T$ are then changed such that $P_{\text{sim}}$ is successively optimized for each element. This procedure is terminated once an entire cycle, where each element in $T$ is considered once, is completed without evoking any changes in $T$.

In the next section we will demonstrate how the marginal and simultaneous SGM classifiers compare to each other. They will also be compared to predictive classifiers based on ordinary GMs as well as the predictive naive Bayes classifier, which is equal to the GM classifier with the empty graph. 

\section{Numerical experiments}
\label{sec:results}

The synthetic data used in the following examples is generated from five different classes. The variables in each class are associated with a unique dependence structure. In each class a group of five variables constitutes a chain component, variables in different chain components are independent of each other. For a given class, the variables in each of the chain components follow the same dependence structure and distribution. This framework makes it possible to easily construct datasets with a larger number of variables by combining any desired amount of chain components. The dependence structure for the variables for each of the five classes follows that of the SGs in Figure \ref{classes}. Note that instead of writing a condition as $(X_1=1, X_2=1)$, in order to save space it is sufficient to write $(1, 1)$, as it is uniquely determined which nodes are adjacent to both nodes in a stratified edge and the variables are ordered topologically. Also a condition where $X_{\delta}=0$ or $X_{\delta}=1$ is written as $X_{\delta}=*$.
\begin{figure}[h]
\begin{center}
\includegraphics{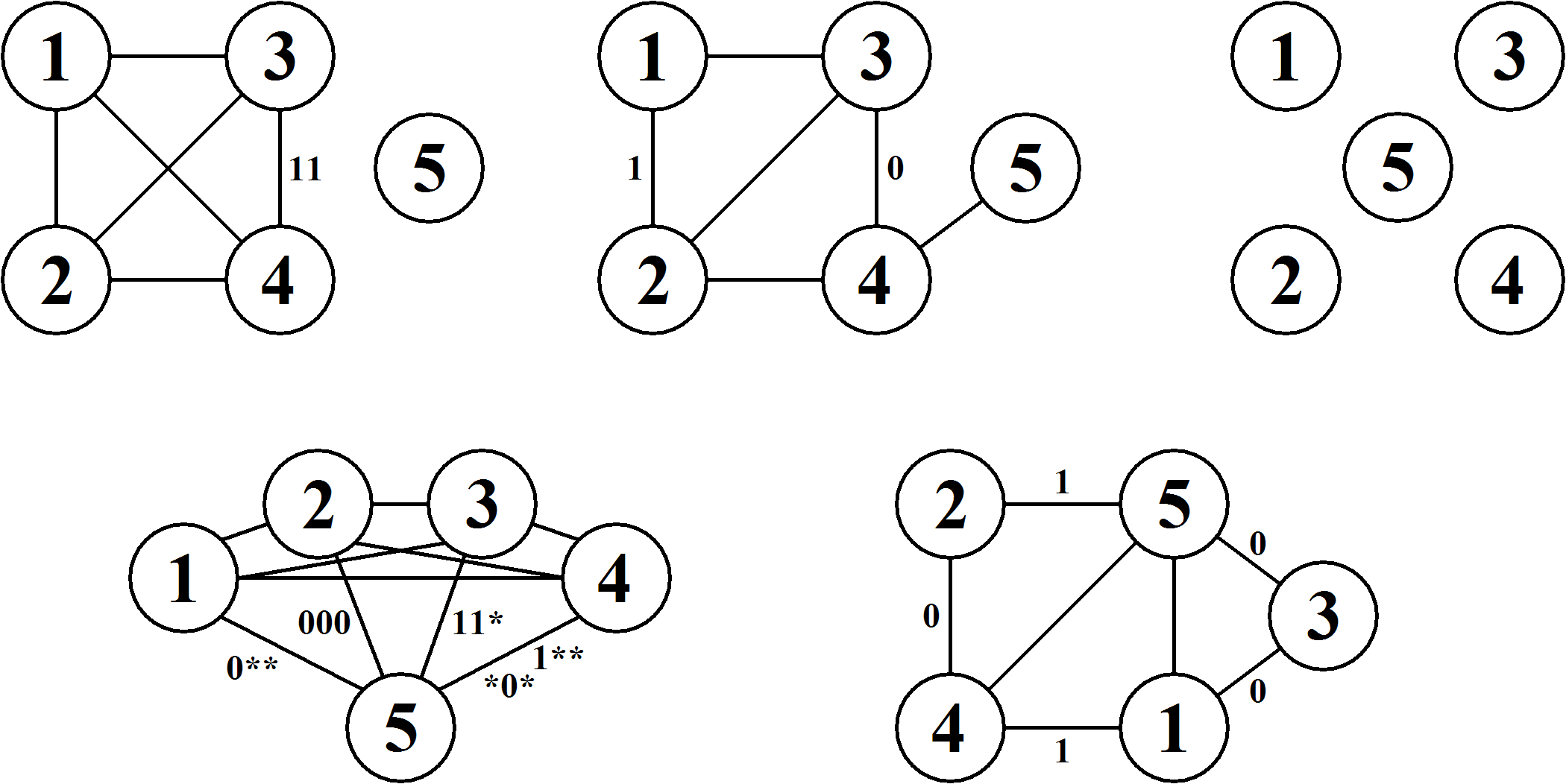}
\end{center}
\caption{Dependence structure for the variables in the five different classes.}
\label{classes}
\end{figure}
The probability distributions used for each class is available as Online Resource 1.

\begin{figure}[htb]
\begin{center}
\includegraphics{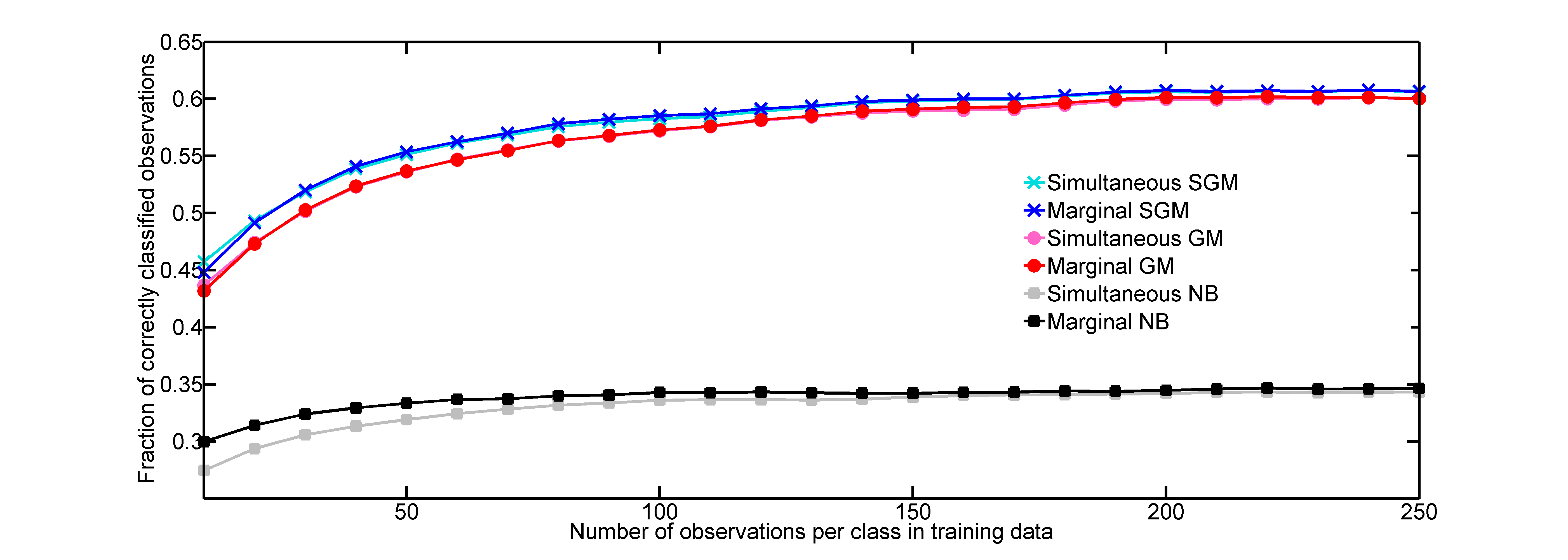}
\end{center}
\caption{Success rates of six classifier methods using 20 features. The number of training data observations per class ranges from 10 to 250, the number of test data observations per class is fixed to 20.}
\label{test20}
\end{figure}

Training data and test data are generated from the five different classes. In the first experiment the number of features is set equal to 20 and we fix the number of observations per class in the test data to 20, while letting the number of observations per class in the training data vary from 10 to 250. Here we make the simplifying assumption that the dependence structure, as encoded by GMs and SGMs, is known for each class. Marginal and simultaneous classifiers were then applied to 200 similarly generated training and test data sets with the average resulting success rates for each classifier displayed in Figure \ref{test20}.

While the differences between the naive Bayes classifiers and the classifiers utilizing the GM or SGM structures are very large, there is also a non-negligible difference between the GM and the SGM classifiers. However, the differences between the marginal and simultaneous classifiers are very small (curves practically overlap) in both the GM and SGM cases. This can be explained by two main reasons. Firstly, the size of the test data is small compared to the training data, meaning that the extra knowledge gained from the test data in the simultaneous case is relatively small. Secondly, the fraction of correctly classified observations is quite low due to the small number of features, meaning that the test data may have an unstabilizing effect on the predictive inference where class-conditional parameters are integrated out.

\begin{figure}[htb]
\begin{center}
\includegraphics{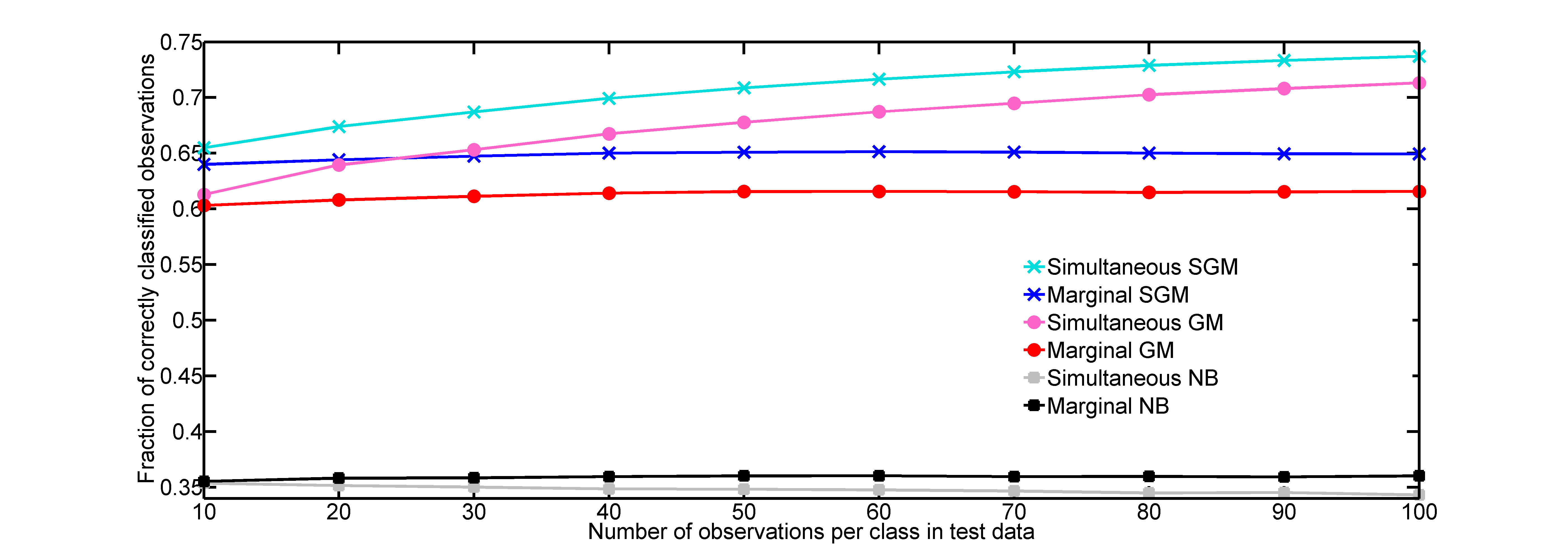}
\end{center}
\caption{The success rates of six different classifier methods using 50 features. The number of test data observations per class ranges from 10 to 100 and the number of training data observations per class is fixed to 20.}
\label{train20}
\end{figure}

In the second experiment we increase the number of features to 50 and let the number of test data observations per class vary from 10 to 100, while fixing the number of training data observations per class to 20. The resulting average success rates for the different classifiers are shown in Figure \ref{train20}. As expected, we see that when the amount of features is increased, the rate of correctly classified observations increases as well. This, in turn, leads to the simultaneous SGM and GM classifiers outperforming their marginal counterparts regardless of the sizes of the training and test data. Moreover, when the dependence structure is known, the SGM classifier perform consistently better than the GM classifier.

\begin{figure}[htb]
\begin{center}
\includegraphics{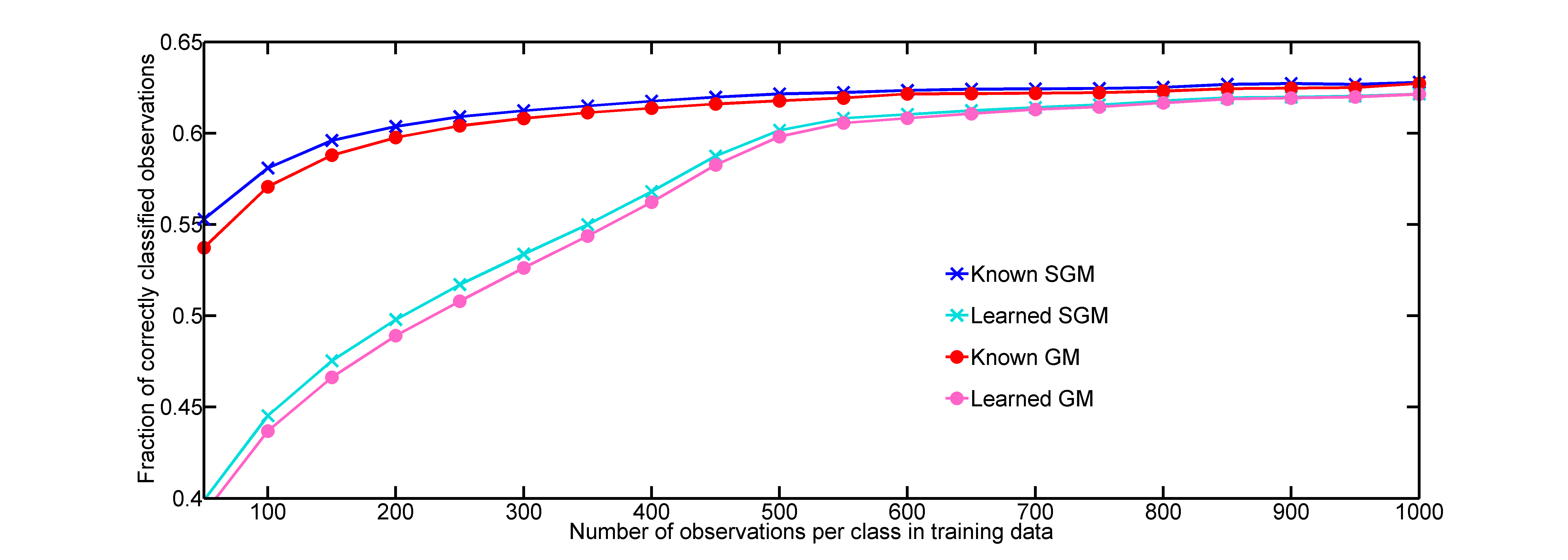}
\end{center}
\caption{The success rates of four different marginal classifier methods. For some classifiers the dependence structure is assumed to be known whilst for others it has to be learned from the training data.}
\label{marginal}
\end{figure}

In the next experiment, to avoid excessive computational demands, we focus on the faster marginal classifiers with the added difficulty that we now also need to learn the dependence structure (GM and SGM) from the training data for each of the 1,000 replicated data sets per training data size. In order to keep the required simulation time tractable, we again limit the number of features to 20. As we in this case only use the marginal classifiers, the size of the test data is irrelevant regarding the computational complexity and resulting success rates. The number of observations per class in the training data is set to vary from 50 to 1000. The results from this test are displayed in Figure \ref{marginal}. The marginal classifiers using the known GM and SGM structures have also been included as a baseline reference. We can immediately see that the classifiers where the structure is learned using the training data have lower success rates than when the structure is known, which is expected. However, we nevertheless see that the SGM classifiers perform better than GM classifiers, irrespectively of whether the dependence structure is learned from the training data or fixed. The results also visualize the asymptotic behavior established in Theorem \ref{th:marSim} and \ref{th:GMSGM}, as the success rates become identical for all the considered classifiers as the size of the training data grows.

Two real datasets are now used to compare performances of the different predictive classifiers. First, the marginal classifiers (naive Bayes, GM, and SGM) are applied to a dataset representing the genetic variation in the bacterial genus Vibrionaceae \citep{Dawyndt05}. The dataset consists of 507 observations on 994 binary fAFLP markers. Using 10 estimation runs of the basic clustering module in BAPS software \citep{Corander06Bayesian, Corander08enhanced} with \textit{a priori} upper bound of the number of clusters ranging between 20-150, the data were split into a total of 80 classes. The ten classes containing most data, between 12-44 observations, were further used in a classification experiment.

Because of the large number of features in the Vibrionaceae data, learning the dependence structure poses a considerable computational challenge. Therefore, to simplify the computations, the features are separated into ten groups, with nine groups of 100 and one group of 94 features. When learning the undirected graph representing the dependence structure in each group and class we also restricted the maximum clique size to five nodes. This restriction has only little effect in practice since the identified structures mostly include small cliques. Once the optimal undirected graphs are obtained, a search for the optimal strata is conducted with the given undirected graph used as the underlying graph. With all these restrictions the resulting GMs and SGMs can by no means be considered to be the corresponding global optima, however, they serve well enough to illustrate the classifier performance.

To assess the accuracy of the different classifiers, the data are randomly split into two sets, the test data containing two observations from each class and the training data comprising the remaining observations. As the test data is relatively small compared to the training data it is reasonable to assume that the marginal and simultaneous classifiers would perform almost identically and therefore only the marginal classifiers are considered. First, all the features were utilized in the classification, resulting in a success rate of 100 percent. While this is an encouraging result it is useless for comparing the different classification methods. Therefore, the classification was then done separately for each of the ten groups of features, which reduces the number of features to 1/10 of the original data set. The resulting average success rates over multiple simulations are displayed in Table \ref{tab:vibrio}.
\begin{table}[htb]
\begin{center}
\begin{tabular}{cccc}
\hline
Group & SGM & GM & naive Bayes \\
\hline
1 & 0.9911 & 0.9870 & 0.9673 \\
2 & 1.0000 & 0.9988 & 0.9208 \\
3 & 0.9982 & 0.9983 & 0.9094 \\
4 & 1.0000 & 1.0000 & 0.9379 \\
5 & 0.9964 & 0.9949 & 0.9134 \\
6 & 0.9908 & 0.9876 & 0.9039 \\
7 & 0.9901 & 0.9884 & 0.9105 \\
8 & 0.9975 & 0.9963 & 0.8624 \\
9 & 0.9228 & 0.9169 & 0.8671 \\
10 & 0.6581 & 0.6414 & 0.5256 \\
\hline
All & 0.9545 & 0.9510 & 0.8718 \\
\hline
\end{tabular}
\end{center}
\caption{Success rates of the marginal classifiers applied to groups of 100 features in the Vibrionaceae data.}
\label{tab:vibrio}
\end{table}

While it is clear that the GM and SGM classifiers have quite similar success rates, the SGM classifier consistently performs better. It is worthwhile to note that the search for the optimal SGM structure is more complicated than that for the GM structure making it more susceptible to errors. Most likely this means that a more extensive search would further favor the SGM classifier. 

The second real dataset that we consider is derived from the answers given by 1444 candidates in the Finnish parliament elections of 2011 in a questionnaire issued by the newspaper Helsingin Sanomat \citep{HelsinginSanomat11}. The dataset, which contains 25 features, is available as Online Resource 2. The candidates belong to one of the eight political parties, listed in Appendix D, whose members were subsequently elected to the parliament. Using the entire data the optimal GM and SGM structures are inferred for each class (political party). Once the dependence structure is known for each class, a single candidate is sequentially classified while the remaining candidates constitute the training data. As the test data in this case consists of a single observation the marginal and simultaneous classifiers are identical. The SGM classifier managed to correctly classify 1084 candidates (75.1$\%$), while the corresponding numbers for the GM and naive Bayes classifiers are 1053 (72.9$\%$) and 974 (67.5$\%$), respectively.
\begin{table}[htb]
\begin{center}
\begin{tabular}{c|cccccccc}
\hline
Political party & 1 & 2 & 3 & 4 & 5 & 6 & 7 & 8 \\
\hline
1 & 163 & 1 & 16 & 10 & 10 & 5 & 0 & 1 \\
2 & 0 & 179 & 0 & 0 & 10 & 2 & 2 & 2 \\
3 & 5 & 0 & 57 & 0 & 1 & 1 & 1 & 2 \\
4 & 13 & 5 & 7 & 138 & 23 & 4 & 3 & 7 \\
5 & 10 & 10 & 1 & 14 & 110 & 11 & 7 & 1 \\
6 & 2 & 14 & 2 & 3 & 13 & 124 & 39 & 11 \\
7 & 0 & 3 & 1 & 4 & 9 & 18 & 143 & 14 \\
8 & 1 & 0 & 5 & 7 & 4 & 9 & 16 & 170 \\
\hline
\end{tabular}
\end{center}
\caption{Resulting class assignment for parliament election data using the SGM classifier.}
\label{tab:hs}
\end{table}
Table \ref{tab:hs} lists how the candidates are assigned to the different parties by the SGM classifier. The element in row $i$ and column $j$ is the number of candidates belonging to party $i$ assigned to party $j$. Parties 1-5 can be considered to be the conservative or moderate parties while parties 6-8 can be considered the more liberal parties. Interestingly, out of the erroneously classified candidates belonging to the conservative parties 136 are assigned to another conservative party, while only 49 are assigned to liberal parties. Similarly, out of the erroneously classified liberal candidates 107 are assigned to other liberal parties and 68 are assigned to conservative parties, respectively.

\section{Discussion}
\label{sec:conclusion}

We introduced a predictive Bayesian classifier that utilizes the dependence structure of the observed features to enhance the accuracy of classification, by allowing a more faithful representation of the data generative process. Albeit we did not consider it explicitly, an additional beneficial characteristic of such an approach is that the uncertainty of the class labels is then more appropriately characterized by the predictive distribution, which may be important in certain applications of sensitive nature and where asymmetric losses are used for erroneous classification decisions across different classes. For a general discussion of this issue, see \cite{Ripley96}. While the naive predictive Bayes classifier is simple and straightforward to use, it often oversimplifies the problem by assuming independence among the features, which has been widely acknowledged in the literature. GM classifiers attempt to rectify the problem by introducing a dependence structure for the features. However, the family of dependence structures that can be modeled using GMs can in some cases be to rigid. The ability to include context-specific independencies among the features, introduced by SGMs, allows for a more precise and sparse representation of the dependence structure.

The results presented in this paper demonstrate the potential of SGM classifiers to improve the rate of success with which the items are classified. Additionally, it is shown that when the data includes a sufficient amount of features, leading to a high success rate of classification, a simultaneous classifier is advantageous compared to the separate classification of each sample which is the standard approach.

In future research it would be interesting to consider SGM classifiers in the context of sequentially arising data, such as discussed in \cite{Corander13b}. Kernel methods (see e.g. \cite{Bishop07pattern}) would also possibly allow a generalization of the context-specific dependence to continuous variables in the sequential case. However, since such methods are generally computation intensive, very efficient fast approximations would need to be used in online type applications, such as in speech recognition and other similar sequential signal processing \citep{Huo00, Maina11}.

\begin{appendix}

\section*{Appendix A: example illustrating the induced effect of strata on the dependence structure}
Consider a clique in a decomposable SG consisting of $d$ nodes. It follows from the definition of decomposable SGs that all stratified edges in this clique have at least one node in common. We introduce an ordering of the variables, corresponding to the nodes of such a clique, such that the last variable in the ordering corresponds to the node found in all stratified edges. This variable is denoted by $X_d$. All the changes to the dependence structure caused by the inclusion of strata can be seen in the conditional distribution $P(X_d \mid X_1, \ldots, X_{d-1})$. In the absence of strata, each outcome of the variables $(X_1, \ldots, X_{d-1})$, termed as parents of $X_d$ and denoted by $\Pi_d$, would induce a unique conditional distribution. However, each outcome in a stratum on an edge serves to merge a subset of these outcomes. For instance, if all variables are considered to be binary, the condition $(X_1 = 0, \ldots, X_{d-2} = 0)$ on the edge $\{d-1, d\}$ will merge the outcomes $(X_1 = 0, \ldots, X_{d-2} = 0, X_{d-1} = 0)$ and $(X_1 = 0, \ldots, X_{d-2} = 0, X_{d-1} = 1)$.

This merging process is best illustrated using conditional probability tables (CPTs). As an example, consider the clique $\{2, 3, 4\}$ in the ordinary graph and SG in Figure \ref{fig:GMSGM}. Variable $X_4$ is by needs chosen to be last in the ordering. Table \ref{tab:CPTs} contains the CPTs for $X_4$ for both the ordinary graph and the SG.
\begin{table}
\begin{tabular}{ccccc}
\hline
Outcome & $X_2$ & $X_3$ & $P(X_4 \mid X_2, X_3)$ in GM & $P(X_4 \mid X_2, X_3)$ in SGM  \\
\hline
(1) & 0 & 0 & $p_1$ & $q_1$ \\
(2) & 0 & 1 & $p_2$ & $q_1$ \\
(3) & 1 & 0 & $p_3$ & $q_1$ \\
(4) & 1 & 1 & $p_4$ & $q_2$ \\
\hline
\end{tabular}
\caption{Corresponding CPTs for $X_4$ in the graph and SG in Figure \ref{fig:GMSGM}.}
\label{tab:CPTs}
\end{table}
The condition $X_2 = 0$ on the edge $\{3, 4\}$ merges outcomes $(1)$ and $(2)$, while the condition $X_3 = 0$ on the edge $\{2, 4\}$ merges outcomes $(1)$ and $(3)$. This results in $X_4$ having only two unique conditional distributions, one for the group of outcomes $(X_2 = 0, X_3 = 0)$, $(X_2 = 0, X_3 = 1)$, and $(X_2 = 1, X_3 = 0)$ and one for the outcome $(X_2 = 1, X_3 = 1)$.

\section*{Appendix B: proof of Theorem \ref{th:marSim}}
To prove Theorem \ref{th:marSim} it suffices to consider a single class $k$ and a single clique in $G_L^k$. If the scores for the marginal and simultaneous classifiers are asymptotically equivalent for an arbitrary clique and class it automatically follows that the scores for the whole system are asymptotically equivalent. We start by considering the simultaneous classifier. The training data $\mathbf{X}^R$ and test data $\mathbf{X}^T$ are now assumed to cover only one clique of an SG in one class. Looking at $\log P_{\text{sim}}(\mathbf{X}^T \mid \mathbf{X}^R)$ using \eqref{eq:ML} we get
\[
\log P_{\text{sim}}(\mathbf{X}^T \mid \mathbf{X}^R)=
\sum_{j = 1}^{d} \sum_{l = 1}^{q_{j}} \log \frac{\Gamma(\sum_{i = 1}^{k_{j}} \beta_{jil})}{\Gamma(n(\pi_{j}^{l}) + \sum_{i = 1}^{k_{j}} \beta_{jil})} +
\sum_{j = 1}^{d} \sum_{l = 1}^{q_{j}} \sum_{i = 1}^{k_{j}} \log \frac{\Gamma(n(x_{j}^{i} \mid \pi_{j}^{l}) + \beta_{jil})}{\Gamma(\beta_{jil})}.
\]
Using Stirling's approximation, $\log \Gamma(x)=(x - 0.5) \log(x) - x$, this equals
\[
\begin{split}
& \sum_{j = 1}^{d} \sum_{l = 1}^{q_{j}} \left( \left(\sum_{i = 1}^{k_{j}} \beta_{jil} - 0.5 \right) \log\left(\sum_{i = 1}^{k_{j}} \beta_{jil} \right) - \sum_{i = 1}^{k_{j}} \beta_{jil} \right) \\
- & \sum_{j = 1}^{d} \sum_{l = 1}^{q_{j}} \left( \left(n(\pi_{j}^{l}) + \sum_{i = 1}^{k_{j}} \beta_{jil} - 0.5 \right) \log\left(n(\pi_{j}^{l}) + \sum_{i = 1}^{k_{j}} \beta_{jil} \right) - n(\pi_{j}^{l}) - \sum_{i = 1}^{k_{j}} \beta_{jil} \right) \\
+ & \sum_{j = 1}^{d} \sum_{l = 1}^{q_{j}} \sum_{i = 1}^{k_{j}} \left( \left(n(x_{j}^{i} \mid \pi_{j}^{l}) + \beta_{jil} - 0.5 \right) \log \left( n(x_{j}^{i} \mid \pi_{j}^{l}) + \beta_{jil} \right) - n(x_{j}^{i} \mid \pi_{j}^{l}) - \beta_{jil} \right) \\
- & \sum_{j = 1}^{d} \sum_{l = 1}^{q_{j}} \sum_{i = 1}^{k_{j}} \left( (\beta_{jil} - 0.5) \log(\beta_{jil}) - \beta_{jil} \right)
\end{split}
\]
\[
\begin{split}
= - & \sum_{j = 1}^{d} \sum_{l = 1}^{q_{j}} \left( \left(\sum_{i = 1}^{k_{j}} \beta_{jil} - 0.5 \right) \log\left(1 + \frac{n(\pi_{j}^{l})}{\sum_{i = 1}^{k_{j}} \beta_{jil}} \right) + n(\pi_{j}^{l}) \log\left(n(\pi_{j}^{l}) + \sum_{i = 1}^{k_{j}} \beta_{jil} \right) \right) \\
+ & \sum_{j = 1}^{d} \sum_{l = 1}^{q_{j}} \sum_{i = 1}^{k_{j}} \left( (\beta_{jil} - 0.5) \log \left(1 + \frac{n(x_{j}^{i} \mid \pi_{j}^{l})}{\beta_{jil}} \right) + n(x_{j}^{i} \mid \pi_{j}^{l}) \log \left( n(x_{j}^{i} \mid \pi_{j}^{l}) + \beta_{jil} \right) \right).
\end{split}
\]
When looking at the marginal classifier we need to summarize over each single observation $\mathbf{X}_h^T$. We use $h(\pi_{j}^{l})$ to denote if the outcome of the parents of variable $X_j$ belongs to group $l$ and $h(x_{j}^{i} \mid \pi_{j}^{l})$ to denote if the outcome of $X_j$ is $i$ given that the observed outcome of the parents belongs to $l$. Observing that $h(\pi_{j}^{l})$ and $h(x_{j}^{i} \mid \pi_{j}^{l})$ are either 0 or 1 we get the result 
\[
\begin{split}
& \log P_{\text{mar}}(\mathbf{X}^T \mid \mathbf{X}^R) = \sum_{h=1}^n \log P(\mathbf{X}_h^T \mid \mathbf{X}^R) \\
= & \sum_{h=1}^n \sum_{j = 1}^{d} \sum_{l = 1}^{q_{j}} \log \frac{\Gamma(\sum_{i = 1}^{k_{j}} \beta_{jil})}{\Gamma(h(\pi_{j}^{l}) + \sum_{i = 1}^{k_{j}} \beta_{jil})} +
\sum_{h=1}^n \sum_{j = 1}^{d} \sum_{l = 1}^{q_{j}} \sum_{i = 1}^{k_{j}} \log \frac{\Gamma(h(x_{j}^{i} \mid \pi_{j}^{l}) + \beta_{jil})}{\Gamma(\beta_{jil})} \\
= - & \sum_{j = 1}^{d} \sum_{l = 1}^{q_{j}} n(\pi_{j}^{l}) \log \left( \sum_{i = 1}^{k_{j}} \beta_{jil} \right) +
\sum_{j = 1}^{d} \sum_{l = 1}^{q_{j}} \sum_{i = 1}^{k_{j}} n(x_{j}^{i} \mid \pi_{j}^{l}) \log (\beta_{jil}).
\end{split}
\]
If we look at the difference $\log P_{\text{sim}}(\mathbf{X}^T \mid \mathbf{X}^R) - \log P_{\text{mar}}(\mathbf{X}^T \mid \mathbf{X}^R)$  we get
\[
\begin{split}
= - & \sum_{j = 1}^{d} \sum_{l = 1}^{q_{j}} \left(\sum_{i = 1}^{k_{j}} \beta_{jil} - 0.5 \right) \log\left(1 + \frac{n(\pi_{j}^{l})}{\sum_{i = 1}^{k_{j}} \beta_{jil}} \right) \\
+ & \sum_{j = 1}^{d} \sum_{l = 1}^{q_{j}} \sum_{i = 1}^{k_{j}} (\beta_{jil} - 0.5) \log \left(1 + \frac{n(x_{j}^{i} \mid \pi_{j}^{l})}{\beta_{jil}}\right) \\
+ & \sum_{j = 1}^{d} \sum_{l = 1}^{q_{j}} \left( n(\pi_{j}^{l}) \log \left( \sum_{i = 1}^{k_{j}} \beta_{jil} \right) - n(\pi_{j}^{l}) \log\left(n(\pi_{j}^{l}) + \sum_{i = 1}^{k_{j}} \beta_{jil} \right) \right) \\
+ & \sum_{j = 1}^{d} \sum_{l = 1}^{q_{j}} \sum_{i = 1}^{k_{j}} \left( n(x_{j}^{i} \mid \pi_{j}^{l}) \log \left( n(x_{j}^{i} \mid \pi_{j}^{l}) + \beta_{jil} \right) - n(x_{j}^{i} \mid \pi_{j}^{l}) \log (\beta_{jil}) \right).
\end{split}
\]
Under the assumption that all the limits of relative frequencies of feature values are strictly positive under an infinitely exchangeable sampling process of the training data, i.e. all hyperparameters  $\beta_{jil} \rightarrow \infty$ when the size of the training data $m \rightarrow \infty$. Using the standard limit $\lim_{y \rightarrow \infty} (1+x/y)^y = e^x$ results in
\[
\begin{split}
& \lim_{m \rightarrow \infty} \log P_{\text{sim}}(\mathbf{X}^T \mid \mathbf{X}^R) - \log P_{\text{mar}}(\mathbf{X}^T \mid \mathbf{X}^R) \\
= - & \sum_{j = 1}^{d} \sum_{l = 1}^{q_{j}} n(\pi_{j}^{l}) + \sum_{j = 1}^{d} \sum_{l = 1}^{q_{j}} \sum_{i = 1}^{k_{j}} n(x_{j}^{i} \mid \pi_{j}^{l}) = 0.
\end{split}
\]
\qed

\section*{Appendix C: proof of Theorem \ref{th:GMSGM}}
This proof follows largely the same structure as the proof of Theorem \ref{th:marSim} and covers the simultaneous score. It is assumed that the underlying graph of the SGM coincides with the GM, this is a fair assumption since when the size of the training data goes to infinity this property will hold for the SGM and GM maximizing the marginal likelihood. Again we consider only a single class $k$ and a single clique in $G_L^k$, using the same reasoning as in the proof above. Additionally, it will suffice to consider the score for the last variable $X_d$ in the ordering, the variable corresponding to the node associated with all of the stratified edges, and a specific parent configuration $l$ of the parents $\Pi_d$ of $X_d$. The equation for calculating the score for variables $X_1, \ldots, X_{d-1}$ will be identical using either the GM or the SGM. If the asymptotic equivalence holds for an arbitrary parent configuration it automatically holds for all parent configurations. Under this setting we start by looking at the score for the SGM
\[
\log P_{\text{SGM}}(\mathbf{X}^T \mid \mathbf{X}^R)=
\log \frac{\Gamma(\sum_{i = 1}^{k_{j}} \beta_{jil})}{\Gamma(n(\pi_{j}^{l}) + \sum_{i = 1}^{k_{j}} \beta_{jil})} +
\sum_{i = 1}^{k_{j}} \log \frac{\Gamma(n(x_{j}^{i} \mid \pi_{j}^{l}) + \beta_{jil})}{\Gamma(\beta_{jil})},
\]
which using Stirling's approximation and the same techniques as in the previous proof equals
\[
\begin{split}
- & \left(\sum_{i=1}^{k_j}\beta_{jil}-0.5\right) \log\left(1 + \frac{n(\pi_j^l)}{\sum_{i=1}^{k_j}\beta_{jil}}\right) - n(\pi_j^l) \log\left(n(\pi_j^l) + \sum_{i=1}^{k_j}\beta_{jil}\right) \\
+ & \sum_{i=1}^{k_j} n(x_j^i \mid \pi_j^l) \log(n(x_j^i \mid \pi_j^l) + \beta_{jil}) + \sum_{i=1}^{k_j} (\beta_{jil} - 0.5) \log\left(1+\frac{n(x_j^i \mid \pi_j^l)}{\beta_{jil}}\right).
\end{split}
\]
When studying the GM score we need to separately consider each outcome in the parent configuration $l$. Let $h$ denote such an outcome in $l$ with the total number of outcomes in $l$ totaling $q_l$. We then get the following score for the GM,
\[
\log P_{\text{GM}}(\mathbf{X}^T \mid \mathbf{X}^R)=
\sum_{h=1}^{q_l} \log \frac{\Gamma(\sum_{i = 1}^{k_{j}} \beta_{jih})}{\Gamma(n(\pi_{j}^{h}) + \sum_{i = 1}^{k_{j}} \beta_{jih})} +
\sum_{h=1}^{q_l} \sum_{i = 1}^{k_{j}} \log \frac{\Gamma(n(x_{j}^{i} \mid \pi_{j}^{h}) + \beta_{jih})}{\Gamma(\beta_{jih})}.
\]
Which, using identical calculations as before, equals
\[
\begin{split}
- & \sum_{h=1}^{q_l} \left(\sum_{i=1}^{k_j}\beta_{jih}-0.5\right) \log\left(1 + \frac{n(\pi_j^h)}{\sum_{i=1}^{k_j}\beta_{jih}}\right) -\sum_{h=1}^{q_l} n(\pi_j^h) \log\left(n(\pi_j^h) + \sum_{i=1}^{k_j}\beta_{jih}\right) \\
+ & \sum_{h=1}^{q_l}\sum_{i=1}^{k_j} n(x_j^i \mid \pi_j^h) \log(n(x_j^i \mid \pi_j^h) + \beta_{jih}) + \sum_{h=1}^{q_l}\sum_{i=1}^{k_j} (\beta_{jih} - 0.5) \log\left(1 + \frac{n(x_j^i \mid \pi_j^h)}{\beta_{jih}}\right).
\end{split}
\]
Considering the difference $\log P_{\text{SGM}}(\mathbf{X}^T \mid \mathbf{X}^R) - \log P_{\text{GM}}(\mathbf{X}^T \mid \mathbf{X}^R)$ we get
\[
\begin{split}
- & \left(\sum_{i=1}^{k_j}\beta_{jil}-0.5\right) \log\left(1 + \frac{n(\pi_j^l)}{\sum_{i=1}^{k_j}\beta_{jil}}\right) +\sum_{i=1}^{k_j} (\beta_{jil} - 0.5) \log\left(1+\frac{n(x_j^i \mid \pi_j^l)}{\beta_{jil}}\right) \\
- & n(\pi_j^l) \log\left(n(\pi_j^l) + \sum_{i=1}^{k_j}\beta_{jil}\right) +\sum_{h=1}^{q_l} n(\pi_j^h) \log\left(n(\pi_j^h) + \sum_{i=1}^{k_j}\beta_{jih}\right) \\
+ & \sum_{i=1}^{k_j} n(x_j^i \mid \pi_j^l) \log(n(x_j^i \mid \pi_j^l) + \beta_{jil}) -\sum_{h=1}^{q_l}\sum_{i=1}^{k_j} n(x_j^i \mid \pi_j^h) \log(n(x_j^i \mid \pi_j^h) + \beta_{jih})  \\
+ & \sum_{h=1}^{q_l} \left(\sum_{i=1}^{k_j}\beta_{jih}-0.5\right) \log\left(1 + \frac{n(\pi_j^h)}{\sum_{i=1}^{k_j}\beta_{jih}}\right) 
-\sum_{h=1}^{q_l}\sum_{i=1}^{k_j} (\beta_{jih} - 0.5) \log\left(1 + \frac{n(x_j^i \mid \pi_j^h)}{\beta_{jih}}\right).
\end{split}
\]
Under the assumption that $\beta_{jil} \rightarrow \infty$ as $m \rightarrow \infty$, the terms in rows one and four will sum to 0 as $m \rightarrow \infty$. The remaining terms can be written
\[
\log\frac{\prod_{h=1}^{q_l} (n(\pi_j^h) + \sum_{i=1}^{k_j}\beta_{jih})^{n(\pi_j^h)}}{(n(\pi_j^l) + \sum_{i=1}^{k_j}\beta_{jil})^{n(\pi_j^l)}}
- \sum_{i=1}^{k_j} \log\frac{\prod_{h=1}^{q_l} (n(x_j^i \mid \pi_j^h) + \beta_{jih})^{n(x_j^i \mid \pi_j^h)}}{(n(x_j^i \mid \pi_j^l) + \beta_{jil})^{n(x_j^i \mid \pi_j^l)}}.
\]
Noting that $n(\pi_j^l) = \sum_{h=1}^{q_l} n(\pi_j^h)$ and $n(x_j^i \mid \pi_j^l) = \sum_{h=1}^{q_l} n(x_j^i \mid \pi_j^h)$ we get
\[
\sum_{h=1}^{q_l} n(\pi_j^h) \log\frac{ n(\pi_j^h) + \sum_{i=1}^{k_j}\beta_{jih}}{n(\pi_j^l) + \sum_{i=1}^{k_j}\beta_{jil}}
- \sum_{i=1}^{k_j} \sum_{h=1}^{q_l} n(x_j^i \mid \pi_j^h) \log\frac{n(x_j^i \mid \pi_j^h) + \beta_{jih}}{n(x_j^i \mid \pi_j^l) + \beta_{jil}}.
\]
By investigating the definition of the $\beta$ parameters in \eqref{eq:beta}, in combination with the fact that the probabilities of observing the value $i$ for variable $X_j$ given that the outcome of the parents is $h$ are identical for any outcome $h$ comprising the group $l$, we get the limits
\[
\lim_{m \rightarrow \infty} \frac{ n(\pi_j^h) + \sum_{i=1}^{k_j}\beta_{jih}}{n(\pi_j^l) + \sum_{i=1}^{k_j}\beta_{jil}} =
\lim_{m \rightarrow \infty} \frac{n(x_j^i \mid \pi_j^h) + \beta_{jih}}{n(x_j^i \mid \pi_j^l) + \beta_{jil}} = \zeta_{jh}.
\]
And subsequently as $m \rightarrow \infty$ the difference $\log P_{\text{SGM}}(\mathbf{X}^T \mid \mathbf{X}^R) - \log P_{\text{GM}}(\mathbf{X}^T \mid \mathbf{X}^R) \rightarrow$
\[
\sum_{h=1}^{q_l} \left(n(\pi_j^h) \log \zeta_{jh} - \sum_{i=1}^{k_j} n(x_j^i \mid \pi_j^h) \log \zeta_{jh} \right) = 
\sum_{h=1}^{q_l} \left(n(\pi_j^h) \log \zeta_{jh} - n(\pi_j^h) \log \zeta_{jh} \right) = 0.
\]
\qed

\section*{Appendix D: list of political parties in the Finnish parliament}
\begin{table}[h]
\begin{center}
\begin{tabular}{cc}
\hline
Label & Political party \\
\hline
1 & National Coalition Party \\
2 & Finns Party \\
3 & Swedish People's Party of Finland \\
4 & Centre Party \\
5 & Christian Democrats of Finland \\
6 & Social Democratic Party of Finland \\
7 & Left Alliance \\
8 & Green League \\
\hline
\end{tabular}
\end{center}
\caption{List of political parties that are members of the Finnish parliament.}
\label{tab:parties}
\end{table}

\end{appendix}

\section*{Acknowledgments}
H.N. and J.P. were supported by the Foundation of \AA bo Akademi University, as part of the grant for the Center of Excellence in Optimization and Systems Engineering. J.X. and J.C. were supported by the ERC grant no. 239784 and academy of Finland grant no. 251170. J.X. was also supported by the FDPSS graduate school.

\bibliographystyle{henrik}
\bibliography{biblio}

\end{document}